\documentclass{llncs}

\usepackage{times}
\usepackage[numbers,sectionbib]{natbib}
\usepackage{amsfonts}
\usepackage{amsmath}
\usepackage[psamsfonts]{amssymb}
\usepackage{amstext}
\usepackage{latexsym}
\usepackage{color}
\usepackage{graphicx}
\usepackage{enumerate}
\usepackage{url}
\usepackage{bbm}
\usepackage{algorithm}
\usepackage{algorithmic}
\usepackage{epsfig}

\def\Rset{\mathbb{R}}

\DeclareMathOperator*{\E}{\rm E}

\newcommand{\sigmaF}{\mathcal{F}}

\newcommand{\set}[1]{\{#1\}}

\newcommand{\mat}[1]{{\mathbf #1}}
\newcommand{\dis}{\mathrm{disc}}
\renewcommand{\L}{{\cal L}}

\newcommand{\cX}{\mathcal{X}}
\newcommand{\cY}{\mathcal{Y}}
\newcommand{\cZ}{\mathcal{Z}}
\newcommand{\cA}{\mathcal{A}}

\renewcommand{\u}{\mat{u}}

\newcommand{\x}{\mat{x}}

\newcommand{\w}{\mat{w}}

\newcommand{\R}{\mathfrak{R}}
\renewcommand{\S}{{\mathcal S}}

\newcommand{\tts}{\small \tt}

\newcommand{\h}{\widehat}
\newcommand{\wt}{\widetilde}
\newcommand{\e}{\epsilon}

\newcommand{\ssigma}{{\boldsymbol \sigma}}
\newcommand{\Mu}{{\boldsymbol \mu}}
\newenvironment{proof*}{\noindent{\bf Proof:}}{}
\newcommand{\ignore}[1]{}

\begin{document}

\title{New Analysis and Algorithm for\\ Learning with Drifting Distributions}
\titlerunning{Learning with Drifting Distributions}

\author{Mehryar Mohri\inst{1,2}
\and Andres Mu\~noz Medina\inst{1}}
\institute{Courant Institute of Mathematical Sciences, New York, NY.
\and Google Research, New York, NY.}

\maketitle

\begin{abstract}
  We present a new analysis of the problem of learning with drifting
  distributions in the batch setting using the notion of
  discrepancy. We prove learning bounds based on the Rademacher
  complexity of the hypothesis set and the discrepancy of
  distributions both for a drifting PAC scenario and a tracking
  scenario. Our bounds are always tighter and in some cases
  substantially improve upon previous ones based on the $L_1$
  distance.  We also present a generalization of the standard on-line
  to batch conversion to the drifting scenario in terms of the
  discrepancy and arbitrary convex combinations of hypotheses.  We
  introduce a new algorithm exploiting these learning guarantees,
  which we show can be formulated as a simple QP. Finally, we report
  the results of preliminary experiments demonstrating the benefits of
  this algorithm.

\end{abstract}


\section{Introduction}

In the standard PAC model \citep{Valiant1984} and other similar theoretical
models of learning \citep{Vapnik1998}, the distribution according to
which training and test points are drawn is fixed over time. However,
for many tasks such as spam detection, political sentiment analysis,
financial market prediction under mildly fluctuating economic
conditions, or news stories, the learning environment is not
stationary and there is a continuous drift of its parameters over
time.

There is a large body of literature devoted to the study of related
problems both in the on-line and the batch learning scenarios. In the
on-line scenario, the target function is typically assumed to be fixed
but no distributional assumption is made, thus input points may be
chosen adversarially \citep{Cesa-BianchiLugosi2006}. Variants of this
model where the target is allowed to change a fixed number of times
have also been studied
\citep{Cesa-BianchiLugosi2006,HerbsterWarmuth1998,HerbsterWarmuth2001,CavallantiCesa-BianchiGentile2007}.
In the batch scenario, the case of a fixed input distribution with a
drifting target was originally studied by Helmbold and Long
\cite{HelmboldLong1994}. A more general scenario was introduced by
Bartlett \cite{Bartlett1992} where the joint distribution over the
input and labels could drift over time under the assumption that the
$L_1$ distance between the distributions in two consecutive time steps
was bounded by $\Delta$. Both generalization bounds and lower bounds
have been given for this scenario \citep{Long1999,BarveLong1997}.  In
particular, Long \cite{Long1999} showed that if the $L_1$ distance
between two consecutive distributions is at most $\Delta$, then a
generalization error of $O((d\Delta)^{1/3})$ is achievable and Barve
and Long \cite{BarveLong1997} proved this bound to be tight.  Further
improvements were presented by Freund and Mansour
\cite{FreundMansour1997} under the assumption of a constant rate of
change for drifting.  Other settings allowing arbitrary but infrequent
changes of the target have also been studied
\cite{BartlettBen-DavidKulkarni2000}. An intermediate model of drift
based on a {\tt \small near} relationship was also recently introduced
and analyzed by \cite{CrammerEven-DarMansourWortman2010} where
consecutive distributions may change arbitrarily, modulo the
restriction that the region of disagreement between nearby functions
would only be assigned limited distribution mass at any time.

This paper deals with the analysis of learning in the presence of
drifting distributions in the batch setting. We consider both the
general drift model introduced by \cite{Bartlett1992} and a related
drifting PAC model that we will later describe. We present new
generalization bounds for both models (Sections~\ref{sec:pac} and
\ref{sec:tracking}). Unlike the $L_1$ distance used by previous
authors to measure the distance between distributions, our bounds are
based on a notion of \emph{discrepancy} between distributions
generalizing the definition originally introduced by
\cite{MansourMohriRostamizadeh2009} in the context of domain
adaptation. The $L_1$ distance used in previous analyses admits
several drawbacks: in general, it can be very large, even in favorable
learning scenarios; it ignores the loss function and the hypothesis
set used; and it cannot be accurately and efficiently estimated from
finite samples (see for example lower bounds on the sample complexity
of testing closeness by \cite{Valiant2011}). In contrast, the
discrepancy takes into consideration both the loss function and the
hypothesis set.

The learning bounds we present in Sections~\ref{sec:pac} and
\ref{sec:tracking} are tighter than previous bounds both because they
are given in terms of the discrepancy which lower bounds the $L_1$
distance, and because they are given in terms of the Rademacher
complexity instead of the VC-dimension. Additionally, our proofs are
often simpler and more concise.  We also present a generalization of
the standard on-line to batch conversion to the scenario of drifting
distributions in terms of the discrepancy measure
(Section~\ref{sec:online}). Our guarantees hold for convex
combinations of the hypotheses generated by an on-line learning
algorithm.  These bounds lead to the definition of a natural
meta-algorithm which consists of selecting the convex combination of
weights in order to minimize the discrepancy-based learning bound
(Section~\ref{sec:algorithm}). We show that this optimization problem
can be formulated as a simple QP and report the results of preliminary
experiments demonstrating its benefits. Finally we will discuss the
practicality of our algorithm in some natural scenarios.

\section{Preliminaries}
\label{sec:prelim}

In this section, we introduce some preliminary notation and key
definitions, including that of the \emph{discrepancy} between
distributions, and describe the learning scenarios we consider.

Let $\cX$ denote the input space and $\cY$ the output space.  We
consider a loss function $L\colon \cY \times \cY \to \Rset_+$ bounded
by some constant $M>0$.  For any two functions $h, h'\colon \cX \to \cY$ and any
distribution $D$ over $\cX\times \cY$, we denote by $\L_D(h)$ the
expected loss of $h$ and by $\L_D(h, h')$ the expected loss of $h$
with respect to $h'$:
\begin{equation}
\L_D(h) = \E_{(x,y) \sim D} [L(h(x),y)]
\qquad \text{and} \qquad
\L_D(h, h') = \E_{x \sim D^1} [L(h(x), h'(x))],
\end{equation}
where $D^1$ is the marginal distribution over $\cX$ derived from $D$.
We adopt the standard definition of the empirical Rademacher
complexity, but we will need the following sequential definition of a
Rademacher complexity, which is related to that of \cite{Rakhlin2010}.
\begin{definition}
\label{def:EmpiricalRademacher}
Let $G$ be a family of functions mapping from a set $\cZ$ to $\Rset$
and $S = (z_1, \ldots, z_T)$ a fixed sample of size $T$ with elements
in $\cZ$.  The \emph{empirical Rademacher complexity} of $G$ for
the sample $S$ is defined by:
\begin{equation}
  \h \R_S(G) = \E_\ssigma \left[\sup_{g \in G}  \frac{1}{T} \sum_{t = 1}^T \sigma_t g(z_t) \right],
\end{equation} 
where $\ssigma = (\sigma_1, \ldots, \sigma_T)^\top$,
with $\sigma_t$s independent uniform random
variables taking values in $\set{-1, +1}$.
The \emph{Rademacher complexity} of $G$ is the
expectation of $\h \R_S(G)$ over all
samples $S = (z_1, \ldots, z_T)$ of size $T$ drawn according to the
product distribution $D = \bigotimes_{t = 1}^TD_t$:
\begin{equation}
\R_T(G) =\E_{S \sim D}[ \h \R_S(G)].
\end{equation}
\end{definition}
Note that this coincides with the standard Rademacher complexity when the
distributions $D_t$, $t \in [1, T]$, all coincide.

A key question for the analysis of learning with a drifting scenario
is a measure of the difference between two distributions $D$ and
$D'$. The distance used by previous authors is the $L_1$
distance. However, the $L_1$ distance is not helpful in this context
since it can be large even in some rather favorable
situations. Moreover, the $L_1$ distance cannot be accurately and
efficiently estimated from finite samples and it ignores the loss
function used. Thus, we will adopt instead the \emph{discrepancy},
which provides a measure of the dissimilarity of two distributions
that takes into consideration both the loss function and the
hypothesis set used, and that is suitable to the
specific scenario of drifting.

Our definition of discrepancy is a generalization to the drifting
context of the one introduced by \cite{MansourMohriRostamizadeh2009} for the analysis of
domain adaptation.  Observe that for a fixed hypothesis $h\in H$, the
quantity of interest with drifting distributions is the difference of
the expected losses $\L_{D'}(h) -\L_D(h)$ for two consecutive
distributions $D$ and $D'$. A natural distance between distributions
in this context is thus one based on the supremum of this quantity
over all $h\in H$.

\begin{definition}
  Given a hypothesis set $H$ and a loss function $L$, the
  $\cY$-\emph{discrepancy} $\dis_\cY$ between two distributions $D$
  and $D'$ over $\cX \times \cY$ is defined by:
\begin{equation}
\dis_\cY(D, D') = \sup_{h \in H} \big| \L_{D'}(h) - \L_{D}(h) \big|.
\end{equation}
\end{definition}
In a deterministic learning scenario with a labeling function
$f$, the previous definition becomes 
\begin{equation}
\dis_\cY(D, D') = \sup_{h \in H} \big| \L_{D'^1}(f, h) - \L_{D^1}(f, h) \big|,
\end{equation}
where $D'^1$ and $D^1$ are the marginal distributions associated to
$D$ and $D'$ defined over $\cX$.  The target function $f$ is unknown
and could match any hypothesis $h'$. This leads to the following
definition \citep{MansourMohriRostamizadeh2009}.
\begin{definition} 
  Given a hypothesis set $H$ and a loss function $L$, the
  \emph{discrepancy} $\dis$ between two distributions $D$ and $D'$
  over $\cX \times \cY$ is defined by:
\begin{equation}
\dis(D, D') = \sup_{h, h' \in H} \big| \L_{D'^1}(h', h) - \L_{D^1}(h', h) \big|.
\end{equation}
\end{definition}
An important advantage of this last definition of discrepancy, in
addition to those already mentioned, is that it can be accurately
estimated from finite samples drawn from $D'^1$ and $D^1$ when the
loss is bounded and the Rademacher complexity of the family of
functions $L_H = \set{x \mapsto L(h'(x), h(x))\colon h, h' \in H}$ is
in $O(1/\sqrt{T})$, where $T$ is the sample size; in particular when $L_H$ has a finite
pseudo-dimension \citep{MansourMohriRostamizadeh2009}.  The discrepancy is by definition
symmetric and verifies the triangle inequality for any loss function
$L$. In general, it does not define a \emph{distance} since we may
have $\dis(D, D') = 0$ for $D' \neq D$.  However, in some cases, for
example for kernel-based hypothesis sets based on a Gaussian kernel,
 the discrepancy has been shown to be a
distance \citep{smooth}.

We will present our learning guarantees in terms of the
$\cY$-discrepancy $\dis_\cY$, that is the most general definition
since guarantees in terms of the discrepancy $\dis$ can be
straightforwardly derived from them. The advantage of the latter
bounds is the fact that the discrepancy can be estimated in that case
from unlabeled finite samples.

We will consider two different scenarios for the analysis of learning
with drifting distributions: the \emph{drifting PAC scenario} and the
\emph{drifting tracking scenario}.

The drifting PAC scenario is a natural extension of the PAC scenario,
where the objective is to select a hypothesis $h$ out of a hypothesis
set $H$ with a small expected loss according to the distribution $D_{T
  + 1}$ after receiving a sample of $T \geq 1$ instances drawn from
the product distribution $\bigotimes_{t = 1}^T D_t$. Thus, the focus in
this scenario is the performance of the hypothesis $h$ with respect to
the environment distribution after receiving the training sample.

The drifting tracking scenario we consider is based on the scenario
originally introduced by \cite{Bartlett1992} for the zero-one loss and
is used to measure the performance of an algorithm $\cA$ (as opposed
to any hypothesis $h$).  In that learning model, the performance of an
algorithm is determined based on its average predictions at each time
for a sequence of distributions. We will generalize its definition by
using the notion of discrepancy and extending it to other loss
functions. The following definitions are the key concepts defining
this model.

\begin{definition}
  For any sample $S = (x_t, y_t)_{t = 1}^T$ of size $T$, we denote by
  $h_{T - 1} \in H$ the hypothesis returned by an algorithm $\cA$
  after receiving the first $T - 1$ examples and by $\h M_T$ its loss
  or mistake on $x_T$: $\h M_T = L(h_{T - 1}(x_T), y_T)$. For a
  product distribution $D = \bigotimes_{t = 1}^TD_t$ on $(\cX \times
  \cY)^T$ we denote by $M_T(D)$ the expected mistake of $\cA$:
 \begin{equation*}
M_T(D) = \E_{S \sim D}[\h M_T] = \E_{S \sim D}[L(h_{T-1}(x_T), y_T)].
\end{equation*}
\end{definition}

\begin{definition} 
\label{def:tracking}
  Let $\Delta>0$ and let $\wt M_T$ be the supremum of $M_T(D)$ over
  all distribution sequences $D = (D_t)$, with $\dis_{\cY}(D_t,
  D_{t+1})<\Delta$. Algorithm $\cA$ is said to
  \emph{$(\Delta,\e)$-track $H$} if there exists $t_0$ such that for
  $T > t_0$ we have $\wt M_T < \inf_{h \in H} \L_{D_T}(h) +\e$.
\end{definition}

An analysis of the tracking scenario with the $L_1$ distance used to
measure the divergence of distributions instead of the discrepancy was
carried out by Long \cite{Long1999} and Barve and Long
\cite{BarveLong1997}, including both upper and lower bounds for $\wt
M_T$ in terms of $\Delta$. Their analysis makes use of an algorithm
very similar to empirical risk minimization, which we will also use in
our theoretical analysis of both scenarios.

\section{Drifting PAC scenario}
\label{sec:pac}

In this section, we present guarantees for the drifting PAC scenario
in terms of the discrepancies of $D_t$ and $D_{T + 1}$ , $t \in [1
,T]$, and the Rademacher complexity of the hypothesis set. We start
with a generalization bound in this scenario and then present a bound
for the agnostic learning setting.

Let us emphasize that learning bounds in the drifting scenario should
of course not be expected to converge to zero as a function of the
sample size but depend instead on the divergence between distributions.

\begin{theorem}
\label{th:generalization}
Assume that the loss function $L$ is bounded by $M$.  Let $D_1,
\ldots, D_{T + 1}$ be a sequence of distributions and let $H_L =
\set{(x,y) \mapsto L(h(x), y)\colon h \in H}$.  Then, for any $\delta >
0$, with probability at least $1 - \delta$, the following holds for
all $h \in H$:
\begin{equation*}
\L_{D_{T + 1}}(h) \leq \frac{1}{T} \sum_{t = 1}^T L(h(x_t), y_t)
+ 2\R_T(H_L) + \frac{1}{T}\sum_{t = 1}^T \dis_\cY(D_t, D_{T + 1}) + M\sqrt{\frac{\log
  \frac{1}{\delta}}{2T}}.
\end{equation*}
\end{theorem}
\begin{proof}
  We denote by $D$ the product distribution $\bigotimes_{t = 1}^T D_t$.
  Let $\Phi$ be the function defined over any sample $S = ((x_1, y_1),
  \ldots, (x_T, y_T)) \in (\cX \times \cY)^T$
  by
\begin{equation*}
\Phi(S) = \sup_{h \in H} \L_{D_{T + 1}}(h) - \frac{1}{T} \sum_{t = 1}^T L(h(x_t), y_t).
\end{equation*}
Let $S$ and $S'$ be two samples differing by one labeled point, say $(x_t,
y_t)$ in $S$ and $(x'_t,y'_t)$ in $S'$, then:
\begin{equation*}
\Phi(S') - \Phi(S) \leq \sup_{h \in H} \frac{1}{T} \Big[ L(h(x'_t), y'_t)
-  L(h(x_t), y_t) \Big] \leq \frac{M}{T}.
\end{equation*}
Thus, by McDiarmid's inequality, the following holds:\footnote{Note
  that McDiarmid's inequality does not require points to be drawn
  according to the same distribution but only that they would be drawn
  independently.}
\begin{equation*}
\Pr_{S \sim D}\Big[\Phi(S) - \E_{S \sim D}[\Phi(S)] > \e\Big] \leq \exp(-2T\e^2/M^2).
\end{equation*}
We now bound $\E_{S \sim D}[\Phi(S)]$ by first rewriting it, as follows:
\begin{align*}
& \E\! \Big[ \sup_{h \in H} \L_{D_{T + 1}}(h) -
  \frac{1}{T} \sum_{t = 1}^T \L_{D_{t}}(h) + \frac{1}{T} \sum_{t =
    1}^T \L_{D_{t}}(h) -
  \frac{1}{T} \sum_{t = 1}^T L(h(x_t), y_t) \Big]&\\
& \leq \!\E\!\Big[ \sup_{h \in H} \L_{D_{T + 1}}(h) \!-\!
  \frac{1}{T} \sum_{t = 1}^T \L_{D_{t}}(h) \Big] 
\!+\! \E\!\Big[\sup_{h \in H}  \frac{1}{T} \sum_{t =
    1}^T \L_{D_{t}}(h) \!-\!
  \frac{1}{T} \sum_{t = 1}^T L(h(x_t), y_t) \Big]\\
& \leq \!\E\! \Big[ \frac{1}{T} \sum_{t = 1}^T \sup_{h \in H} \big( \L_{D_{T + 1}}(h) -
  \L_{D_{t}}(h)  \big) \!+\! \sup_{h \in H}  \frac{1}{T} \sum_{t =
    1}^T \big( \L_{D_{t}}(h) - L(h(x_t), y_t) \big) \Big]\\
& \leq \frac{1}{T} \sum_{t = 1}^T \dis_\cY(D_t, D_{T + 1}) + \E\! \Big[ \sup_{h \in H}  \frac{1}{T} \sum_{t =
    1}^T \big( \L_{D_{t}}(h) - L(h(x_t), y_t) \big) \Big].
\end{align*}
It is not hard to see, using a symmetrization argument as in
the non-sequential case, that the second term can be bounded by
$2\R_T(H_L)$. \qed
\end{proof}

For many commonly used loss functions, the empirical Rademacher
complexity $\R_T(H_L)$ can be upper bounded in terms of that of the
function class $H$. In particular, for the zero-one loss it is known
that $\R_T(H_L) = \R_T(H)/2$ and when $L$ is the $L_q$ loss for some
$q \geq 1$, that is $L(y, y') = |y' - y|^q$ for all $y, y' \in \cY$,
then $\R_T(H_L) \leq q M^{q - 1} \R_T(H)$. Indeed, since $x \mapsto
|x|^q$ is $q M^{q - 1}$-Lipschitz over $[-M, +M]$, by Talagrand's
contraction lemma, $\R_T(H_L)$ is bounded by $q M^{q - 1} \h \R_T(G)$
with \linebreak $G = \set{(x, y) \mapsto (h(x) - y) \colon h \in H}$.
Furthermore, $\h \R_T(G)$ can be analyzed as follows:
\begin{multline*}
\h \R_T(G)  = \frac{1}{T} \E_\ssigma \bigg[
\sup_{h \in H} \sum_{t = 1}^T \sigma_t (h(x_t) - y_t) \bigg] \\
\begin{aligned}
& = \frac{1}{T} \E_\ssigma \bigg[ \sup_{h \in H} \sum_{t = 1}^T \sigma_t
h(x_t) \bigg] + \frac{1}{T}\E_\ssigma \bigg[ \sum_{t = 1}^T -\sigma_t
y_t \bigg] 
= \h \R_T(H),
\end{aligned}
\end{multline*}
since $\E_\ssigma [ \sum_{t = 1}^T -\sigma_t
y_t]  = 0$. Taking the
expectation of both sides yields a similar inequality for Rademacher
complexities.  Thus, in the statement of the previous theorem,
$\R_T(H_L)$ can be replaced with $q M^{q - 1} \R_T(H)$ when $L$ is the $L_q$
loss.

Observe that the bound of Theorem~\ref{th:generalization} is tight
as a function of the divergence measure (discrepancy) we are using.
Consider for example the case where $D_1 = \ldots = D_T$, then
a standard Rademacher complexity generalization bound holds for
all $h \in H$:
\begin{equation*}
  \L_{D_T}(h) \leq \frac{1}{T} \sum_{t  = 1}^T L(h(x_t), y_t) + 2 \R_T(H_L) + O(1/\sqrt{T}).
\end{equation*}
Now, our generalization bound for $\L_{D_{T + 1}}(h)$ includes only
the additive term \linebreak $\dis_\cY(D_t, D_{T + 1})$, but by definition of the
discrepancy, for any $\e > 0$, there exists $h \in H$ such that the
inequality $|\L_{D_{T + 1}}(h) - \L_{D_T}(h)| < \dis_\cY(D_t, D_{T + 1}) +
\e$ holds.
      
Next, we present PAC learning bounds for empirical risk minimization.
Let $h_T^*$ be a best-in class hypothesis in $H$, that is one with the
best expected loss. By a similar reasoning as in theorem
\ref{th:generalization}, we can show that with probability $1 -
\frac{\delta}{2}$ we have
\begin{equation*}
\frac{1}{T} \!\!\sum_{t=1}^T \! L(h^*_T(x_t),y_t) \!\leq\! \L_{D_{T+1}}(h^*_T) +
2\R_T(H_L) +
\frac{1}{T} \!\sum_{t = 1}^T \!\!\dis_\cY(D_t,D_{T+1}) + 2M\sqrt{\frac{\log \frac{2}{\delta}}{2T}}.
\end{equation*}
Let $h_T$ be a hypothesis returned by empirical risk minimization
(ERM).  Combining this inequality with the bound of theorem
\ref{th:generalization} while using the definition of $h_T$ and using
the union bound, we obtain that with probability $1 - \delta$ the
following holds:
\begin{equation}
\label{agnosticineq}
\L_{D_{T+1}}(h_T)-\L_{D_{T+1}}(h_T^*)\leq 4\R_T(H_L) 
+ \frac{2}{T}\sum_{t =
  1}^T\dis_\cY(D_t,D_{T+1})
+ 2M\sqrt{\frac{\log \frac{2}{\delta}}{2T}}.
\end{equation}
This learning bound indicates a trade-off: larger values of the sample
size $T$ guarantee smaller first and third terms; however, as $T$
increases, the average discrepancy term is likely to grow as well,
thereby making learning increasingly challenging. This suggests an
algorithm similar to empirical risk minimization but limited to the
last $m$ examples instead of the whole sample with $m < T$. This
algorithm was previously used in \cite{BarveLong1997} for the study of
the tracking scenario. We will use it here to prove several theoretical
guarantees in the PAC learning model.

\begin{proposition}
\label{prop:8}
Let $\Delta \geq 0$. Assume that $(D_t)_{t \geq 0}$ is a sequence of
distributions such that $\dis_\cY(D_t,D_{t+1})\leq \Delta$ for all $t
\geq 0$. Fix $m \geq 1$ and let $h_T$ denote the hypothesis returned
by the algorithm $\cA$ that minimizes $\sum_{t=T-m}^T L(h(x_t),y_t)$
after receiving $T > m$ examples. Then, for any $\delta > 0$, with
probability at least $1 - \delta$, the following learning bound holds:
\begin{equation}
\label{eq:agnosticbound}
     \L_{D_{T+1}}(h_T) - \inf_{h \in H} \L_{D_{T+1}}(h) 
\leq 4\R_{m}(H_L) + (m+1)\Delta + 2M\sqrt{\frac{\log \frac{2}{\delta}}{2m}}.
 \end{equation}
\end{proposition}
\begin{proof}	
  The proof is straightforward. Notice that the algorithm discards the
  first $T - m$ examples and considers exactly $m$ instances. Thus, as
  in inequality \ref{agnosticineq}, we have:
\begin{equation*}
\L_{D_{T+1}}(h_T)-\L_{D_{T+1}}(h_T^*)\leq 4\R_{m}(H_L) 
+ \frac{2}{m}\sum_{t=T-m}^T \dis(D_t,D_{T+1}) + 2M\sqrt{\frac{\log
    \frac{2}{\delta}}{2m}}.
\end{equation*}
Now, we can use the triangle inequality to bound $\dis (D_t, D_{T + 1})$ by
 $(T + 1 - m) \Delta$. Thus, the sum of the discrepancy terms can be
 bounded by $(m + 1)\Delta$.
\qed
\end{proof}
To obtain the best learning guarantee, we can select $m$ to minimize
the bound just presented. This requires the expression of the
Rademacher complexity in terms of $m$. The following is the result
obtained when using a VC-dimension upper bound of
$O(\sqrt{d/m})$ for the Rademacher complexity.

\begin{corollary}
\label{cor:cubicpac}
Fix $\Delta > 0$. Let $H$ be a hypothesis set with VC-dimension $d$
such that for all $m \geq 1$, $\R_m(H_L)\leq
\frac{C}{4}\sqrt{\frac{d}{m}}$ for some constant $C > 0$.  Assume that
$(D_t)_{t>0}$ is a sequence of distributions such that
$\dis_{\cY}(D_t,D_{t+1})\leq \Delta$ for all $t \geq 0$. Then, there
exists an algorithm $\cA$ such that for any $\delta > 0$, the
hypothesis $h_T$ it returns after receiving $T > \Big[\frac{C + C'}{2}\Big]^{\frac 2 3}
(\frac{d}{\Delta^2})^{\frac 1 3}$ instances, where $C' = 2M
    \sqrt{\frac{\log(\frac{2}{\delta})}{2d}}$, satisfies the following
with probability at least $1 - \delta$:
\begin{equation}
\label{eq:9}
\L_{D_{T+1}}(h_T) - \inf_{h \in H} \L_{D_{T+1}}(h)  \leq 3 \left[\frac{C + C'}{2}\right]^{2/3}(d\Delta)^{1/3}
+ \Delta.
\end{equation}
\end{corollary}
\begin{proof*}
  Fix $\delta > 0$. Replacing $\R_m(H_L)$
  by the upper bound $\frac{C}{4}\sqrt{\frac{d}{m}}$ in
  \eqref{eq:agnosticbound} yields
\begin{equation*}
  \L_{D_{T+1}}(h_T) - \inf_{h \in H} \L_{D_{T+1}}(h) \leq (C + C') \sqrt{\frac{d}{m}} + (m + 1) \Delta.
\end{equation*}
Choosing $m = (\frac{C + C'}{2})^{\frac 2 3}
(\frac{d}{\Delta^2})^{\frac 1 3} $ to minimize the right-hand side
gives exactly \eqref{eq:9}. \qed
\end{proof*}
When $H$ has finite VC-dimension $d$, it is known that $\R_m(H_L)$
can be bounded by $C\sqrt{d/m}$ for some constant $C > 0$, by using  
a chaining argument \citep{Dudley84,Pollard84,Talagrand2005}. 
Thus, the assumption of the corollary holds for many loss functions
$L$, when $H$ has finite VC-dimension.

\section{Drifting Tracking scenario}
\label{sec:tracking}

In this section, we present a simpler proof of the bounds given by
\cite{Long1999} for the agnostic case demonstrating that using the
discrepancy as a measure of the divergence between distributions leads
to tighter and more informative bounds than using the $L_1$ distance.

\begin{proposition}

  Let $\Delta>0$ and let $(D_t)_{t \geq 0}$ be a sequence
  of distributions such that $\dis_\cY(D_t,D_{t+1})\leq \Delta$ for
  all $t \geq 0$. Let $m>1$ and let $h_T$ be as in proposition
 \ref{prop:8}. Then,
 \begin{equation}
 \label{eq:trackingrademacher}
  \E_D[\h M_{T+1}] - \inf_h\L_{D_{T+1}}(h)
\leq 4\R_{m}(H_L) + 2M\sqrt{\frac{\pi}{m}} + (m + 1)\Delta.
 \end{equation}
\end{proposition}

\begin{proof}
  Let $D = \bigotimes_{t = 1}^{T + 1}D_t$ and $D'=\bigotimes_{t = 1}^{T} D_t$. 
  By Fubini's theorem we can write:
\begin{equation}
\label{eq:11}
  \E_D[\h M_{T+1}] - \inf_h\L_{D_{T+1}}(h) =
  \E_{D'}\Big[\L_{D_{T+1}}(h_T) - \inf_h\L_{D_{T+1}}(h)\Big].
\end{equation} 
Now, let $\phi^{-1}(\delta)=4\R_{m}(H_L)+(m+1)\Delta+2M\sqrt{\frac{\log \frac{2}{\delta}}{2m}}$,
then, by \eqref{eq:agnosticbound}, for \linebreak $\beta>4\R_{m}(h)+(m+1)\Delta$, the
following holds:
\begin{equation*}
\Pr_{D'}[\L_{D_{T+1}}(h_T)-\inf_h\L_{D_{T+1}}(h)>\beta] < \phi(\beta).
\end{equation*}
Thus, the expectation on the right-hand side of \eqref{eq:11} can be
bounded as follows:
\begin{multline*}
\E_{D'}\Big[\L_{D_{T+1}}(h_T) - \inf_h\L_{D_{T+1}}(h)\Big]
\leq 4\R_{m}(H_L) + (m + 1)\Delta + \int_{4\R_{m}(H_L)+(m+1)\Delta}^\infty
\mspace{-40mu} \phi(\beta) d\beta.
\end{multline*}
The last integral can be rewritten as $2M \int_0^2\frac{d\delta}{\sqrt{m\log
    \frac{2}{\delta}}} = 2M\sqrt{\frac{\pi}{m}}$ using the change of variable
$\delta = \phi(\beta)$. This concludes the proof.\qed
\end{proof}
The following corollary can be shown using the same proof as that of
corollary~\ref{cor:cubicpac}.
\begin{corollary}
  Fix $\Delta > 0$. Let $H$ be a hypothesis set with VC-dimension $d$
  such that for all $m > 1$, $4\R_m(H_L)\leq C\sqrt{\frac{d}{m}}$. Let
  $(D_t)_{t>0}$ be a sequence of distributions over $\cX \times \cY$
  such that $\dis_{\cY}(D_t,D_{t+1})\leq \Delta$. Let $C' =
  2M\sqrt{\frac{\pi}{d}}$ and $K = 3 \left[\frac{C +
      C'}{2}\right]^{2/3}$.  Then, for $T > \big[\frac{C +
    C'}{2}\big]^{\frac 2 3} (\frac{d}{\Delta^2})^{\frac 1 3}$,
  the following inequality holds:
\begin{equation*}
\E_D[\h M_{T+1}] - \inf_h\L_{D_{T+1}}(h) < K(d\Delta)^{1/3} + \Delta.
\end{equation*}
\end{corollary}
In terms of definition \ref{def:tracking}, this corollary shows 
that algorithm $\cA$  $(\Delta,K(d\Delta)^{1/3} + \Delta)$-tracks $H$. 
This result is similar to a result of \cite{Long1999} which states that
given $\e > 0$ if $\Delta = O(d\epsilon^3)$ then $\cA$ 
$(\Delta,\epsilon)$-tracks $H$. However, in \cite{Long1999}, 
$\Delta$ is an upper bound on the $L_1$ distance and not
the discrepancy.
\ignore{Of course for $h_T$ to make sense we need $T > m$,
 but since definition \ref{def:tracking}
only requires the condition to be satisfied for big enough
$T$ we have no problems with assuming $T>m$.}
Our result provides thus a tighter and more general guarantee than
that of \citep{Long1999}, the latter because this result is applicable
to any loss function and not only the zero-one loss, the former because
our bound is based on the Rademacher complexity instead of the
VC-dimension and more importantly because it is based on the
discrepancy, which is a finer measure of the divergence between
distributions than the $L_1$ distance. Indeed, for any $t \in [1, T]$,
\begin{align*}
\dis_\cY(D_t, D_{t + 1}) 
& = \sup_{h \in H} \big| \L_{D_t}(h) - \L_{D_{t + 1}}(h)
\big|\\
& = \sup_{h \in H} \big| \sum_{x, y} (D_t(x, y) - D_{t + 1}(x, y))
L(h(x), y) \big|
\big|\\
& \leq M \sup_{h \in H} \sum_{x, y} |D_t(x,y) - D_{t + 1}(x, y)| = M L_1(D_t,
D_{t + 1}).
\end{align*}
Furthermore, when the target function $f$ is in $H$, then the
$\cY$-discrepancies can be bounded by the discrepancies $\dis(D_t,
D_{T + 1})$, which, unlike the $L_1$ distance, can be accurately
estimated from finite samples.

It is important to emphasize that even though our analysis was based
on a particular algorithm, that of ``truncated'' empirical risk
minimization, the bounds obtained here cannot be improved upon in the
general scenario of drifting distributions, as shown by
\cite{BarveLong1997} in the case of binary classification.

\begin{figure}[t]
\begin{center}
\includegraphics[scale=.33]{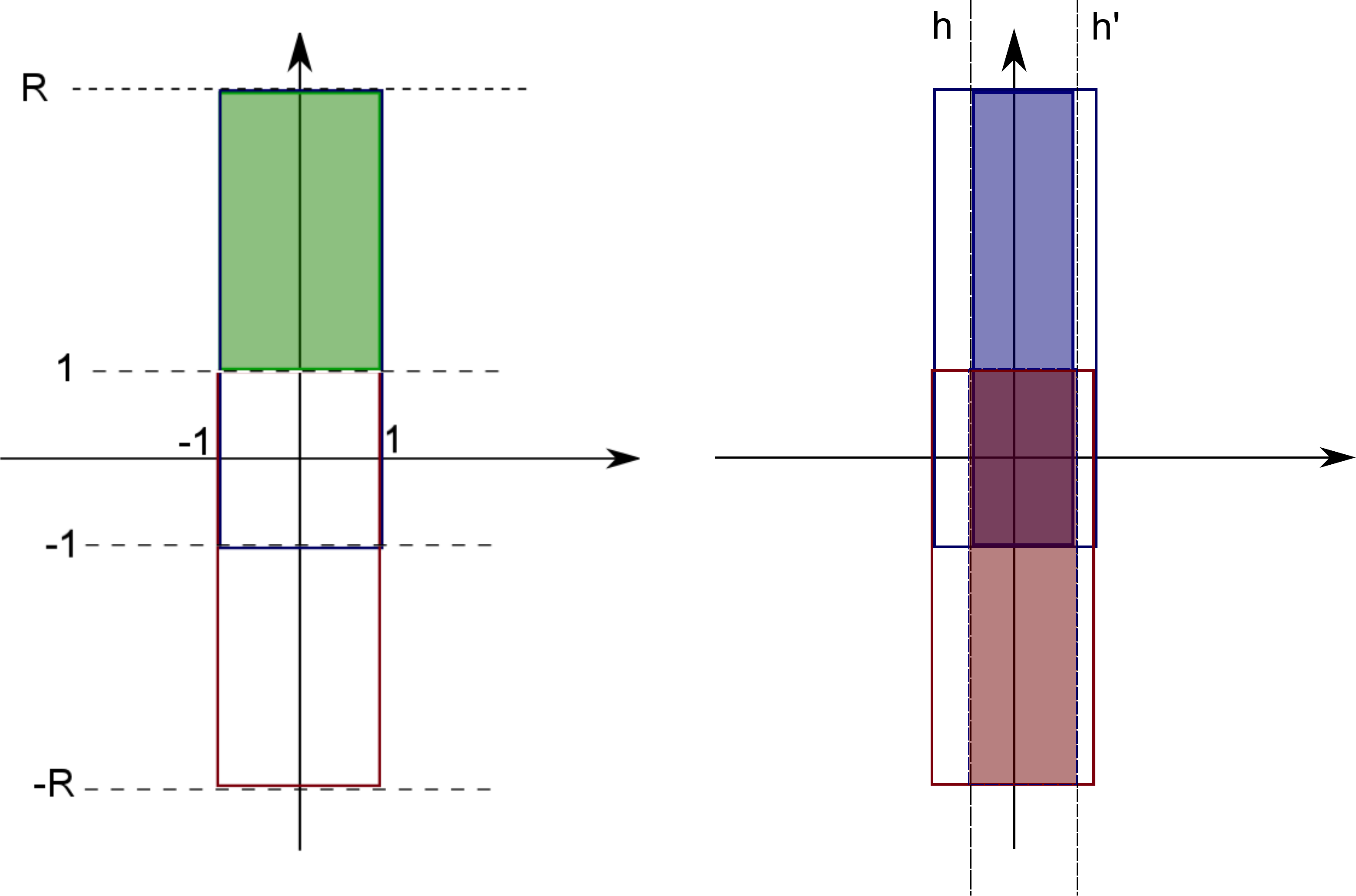}
\end{center}
\vspace{-.75cm}
\caption{Figure depicting the difference between the $L_1$ distance
  and the discrepancy. In the left figure, the $L_1$ distance is given
  by twice the area of the green rectangle. In the right figure, $P(h(x)
  \neq h'(x))$ is equal to the area of the blue rectangle and $Q(h(x)
  \neq h'(x))$ is the area of the red rectangle.  The two areas are
  equal, thus $\dis(P, Q) = 0$.}
\vspace{-.5cm}
\label{fig:discrepancy}
\end{figure}

We now illustrate the difference between the guarantees we present and
those based on the $L_1$ distance by presenting a simple example for the
zero-one loss where the $L_1$ distance can be made arbitrarily close
to 2 while the discrepancy is 0. In that case, our bounds state that
the learning problem is as favorable as in the absence of any
drifting, while a learning bound with the $L_1$ distance would be
uninformative. Consider measures $P$ and $Q$ in $\mathbb{R}^2.$ Where
$P$ is uniform in the rectangle $R_1$ defined by the vertices
$(-1,R),\,(1,R),\, (1,-1),\,(-1,-1)$ and $Q$ is uniform in the
rectangle $R_2$ spanned by $(-1,-R),\,(1,-R),\, (-1,1),\,(1,1)$. The
measures are depicted in figure \ref{fig:discrepancy}. The $L_1$
distance of these probability measures is given by twice the
difference of measure in the green rectangle, i.e, $|P - Q| =
2\frac{(R-1)}{R+1}$ this distance goes to $2$ as $R\rightarrow
\infty$. On the other hand consider the zero-one loss and the
hypothesis set consisting of threshold functions on the first
coordinate, i.e. $h(x, y) = 1$ iff $h < x$. For any two hypotheses
$h<h'$ the area of disagreement of this two hypotheses is given by the
stripe $S = \set{x\colon h < x <h'}$.  But it is trivial to see that
$P(S) = P(S \cap R_1) = (h - h')/2$, but also $Q(S) = Q(S \cap R_2) =
(h - h')/2$, since this is true for any pair of hypotheses we conclude
that $\dis(P, Q) = 0$. This example shows that the learning bounds we
presented can be dramatically more favorable than those given in the
past using the $L_1$ distance.

Although this may be viewed as a trivial illustrative example, the
discrepancy and the $L_1$ distance can greatly differ in more complex
but realistic cases.\ignore{ In the previous example this happened
  because the hypothesis set just depended on one coordinate, but
  similar results could be obtained when there are features that
  aren't really significant for finding a good hypothesis.}

\section{On-line to batch conversion}
\label{sec:online}

In this section, we present learning guarantees for drifting
distributions in terms of the regret of an on-line learning algorithm
$\cA$. The algorithm processes a sample $(x_t)_{t \geq 1}$
sequentially by receiving a sample point $x_t \in \cX$, generating a
hypothesis $h_t$, and incurring a loss $L(h(x_t), y_t)$, with $y_t \in
\cY$.  We denote by $R_T$ the regret of algorithm $\cA$ after
processing $T \geq 1$ sample points:
\begin{equation*}
R_T = \sum_{t = 1}^T L(h(x_t), y_t) - \inf_{h \in H} \sum_{t = 1}^T
L(h(x_t), y_t).
\end{equation*}
The standard setting of on-line learning assumes an adversarial
scenario with no distributional assumption. Nevertheless, when the
data is generated according to some distribution, the hypotheses
returned by an on-line algorithm $\cA$ can be combined to define a
hypothesis with strong learning guarantees in the distributional
setting when the regret $R_T$ is in $O(\sqrt{T})$ (which is attainable
by several regret minimization algorithms)
\citep{Littlestone,Cesa-BianchiConconiGentile2001}. Here, we extend
these results to the drifting scenario and the case of a convex
combination of the hypotheses generated by the algorithm.
The following lemma will be needed for the proof of our main result.

\ignore{
\begin{lemma}
\label{lm:onlineiid}
Let $(h_t)_{t = 1}^T$ be as in the previous definition and assume that
the sample $(x_t, y_t) _{t = 1}^T$ is drawn i.i.d.\ form some
distribution $D$. Then, if the loss function $L$ is bounded by $M$ and
is convex with respect to its first argument, then, for any $\delta >
0$, with probability at least $1 - \delta$, each of the following
learning guarantees holds for the hypothesis $h = \frac{1}{T}\sum_{t =
  1}^T h_t$:
\begin{align}
\label{1giid}
& \L_D(h)\leq \frac{1}{T}\sum_{t = 1}^T L(h_t(x_t),y_t) +M\sqrt{\frac{2\log\frac{1}{\delta}}T}\\
\label{2giid}
& \L_D(h)\leq \inf_{h'\in H}\L(h') +\frac{R_T}{T}+3M\sqrt{\frac{\log\frac{2}{\delta}}{2T}}.
\end{align}
\end{lemma} 
When $R_T = O(\sqrt{T})$, the bound \eqref{2giid} is similar to the
learning guarantees known for PAC learning. We will prove that
similars result to that of theorem~\ref{th:generalization} for the
setting of drifting distributions.  
}

\begin{lemma}
\label{lemma:13}
  Let $\S = (x_t, y_t)_{t = 1}^T$ be a sample drawn from the
  distribution $D = \bigotimes D_t$ and let $(h_t)_{t = 1}^T$ be the
  sequence of hypotheses returned by an on-line algorithm sequentially
  processing $\S$. Let $\w = (w_1, \ldots, w_t)^\top$ be a vector of
  non-negative weights verifying $\sum_{t = 1}^T w_t = 1$. If the loss
  function $L$ is bounded by $M$ then, for any $\delta > 0$, with
  probability at least $1 - \delta$, each of the following inequalities hold:
\begin{align*}
& \sum_{t = 1}^T w_t\L_{D_{T+1}}(h_t)  \leq \sum_{t = 1}^T 
  w_t L(h_t(x_t), y_t) + \bar \Delta(\w, T) 
 + M\| \w  \|_2\sqrt{2\log \frac{1}{\delta}}\\
& \sum_{t = 1}^T w_t L(h_t(x_t),y_t)  \leq \sum_{t = 1}^T 
  w_t\L_{D_{T+1}}(h_t) + \bar \Delta(\w, T)
 + M\|  \w  \|_2\sqrt{2\log \frac{1}{\delta}},
\end{align*}
where $\bar \Delta(\w, T)$ denotes the average discrepancy $\sum_{t =
  1}^T w_t \dis_\cY(D_t,D_{T+1})$.
\end{lemma}
\begin{proof}
  Consider the random process: $Z_t=w_t L(h_t(x_t), y_t) - w_t\L(h_t)$
  and let $\sigmaF_t$ denote the filtration associated to the sample
  process. We have: $|Z_t|\leq Mw_t$ and
\begin{equation*}
\E_{D}[Z_t|\sigmaF_{t-1}] = \E_{D}[w_tL(h_t(x_t), y_t)|\sigmaF_{t-1}] - \E_{D_t}[w_tL(h_t(x_t), y_t)] = 0
\end{equation*}
The second equality holds because $h_t$ is determined at time $t - 1$
and $x_t, y_t$ are independent of $\sigmaF_{t-1}$. Thus, by
Azuma-Hoeffding's inequality, for any $\delta > 0$, with probability
at least $1 - \delta$ the following holds:
\begin{equation}
\label{firstbound}
\sum_{t = 1}^T w_t\L_{D_t}(h_t) \leq \sum_{t = 1}^Tw_tL(h(x_t),y_t)+M\|
\w \|_2\sqrt{2\log \frac{1}{\delta}}.
\end{equation}
By definition of the discrepancy, the following inequality holds for
any $t \in [1, T]$:
\begin{equation*}
\L_{D_{T+1}}(h_t)\leq \L_{D_t}(h_t)+\dis_\cY(D_t, D_{T+1}).
\end{equation*} 
Summing up these inequalities and using \eqref{firstbound} to bound 
$\sum_{t = 1}^T w_t\L_{D_t}(h_t)$ proves the first statement. The
second statement can be proven in a similar way.
\qed\end{proof}
The following theorem is the main result of this section.
\begin{theorem} 
\label{th:onlinetobatch}
Assume that $L$ is bounded by $M$ and convex with respect to its first
argument.  Let $h_1, \ldots, h_T$ be the hypotheses returned by $\cA$
when sequentially processing $(x_t, y_t) _{t = 1}^T$ and let $h$ be
the hypothesis defined by $h = \sum_{t = 1}^T w_t h_t$, where $w_1,
\ldots, w_T$ are arbitrary non-negative weights verifying $\sum_{t =
  1}^T w_t = 1$. Then, for any $\delta > 0$, with probability at least
$1 - \delta$, $h$ satisfies each of the following learning guarantees:
\begin{align*}
& \L_{D_{T+1}}(h)  \leq \sum_{t = 1}^T w_tL(h_t(x_t),y_t) + \bar
\Delta(\w, T)  + M\| \w \|_2\sqrt{2 \log \frac{1}{\delta}}\\
& \L_{D_{T+1}}(h)
\leq \inf_{h \in H} \L(h) +  \frac{R_T}{T} + \bar \Delta(\w, T)  + M  \| \w - \u_0 \|_1 
+ 2M \| \w \|_2 \sqrt{2\log \frac{2}{\delta}},
\end{align*}
where $\w = (w_1, \ldots, w_T)^\top$, $\bar \Delta(\w, T) = \sum_{t = 1}^T w_t
  \dis_\cY(D_t,D_{T+1})$,  and $\u_0 \in \Rset^T$ is
the vector with all its components equal to $1/T$.
\end{theorem}
Observe that when all weights are all equal to $\frac{1}{T}$, the
result we obtain is similar to the learning guarantee obtained in
theorem~\ref{th:generalization} when the Rademacher complexity of
$H_L$ is $O(\frac{1}{\sqrt{T}})$. Also, if the learning scenario is
i.i.d., then the first sum of the bound vanishes and it can be seen
straightforwardly that to minimize the RHS of the inequality we need
to set $w_t = \frac{1}{T}$, which results in the known i.i.d.\
guarantees for on-line to batch conversion
\citep{Littlestone,Cesa-BianchiConconiGentile2001}.  
\begin{proof}
  Since $L$ is convex with respect to its first argument, 
  by Jensen's inequality, we have $\L_{D_{T+1}}(\sum_{t = 1}^T
  w_th_t)\leq \sum_{t = 1}^T w_t\L_{D_{T+1}}(h_t)$. Thus, by
  Lemma~\ref{lemma:13}, for any $\delta > 0$, the following holds with
  probability at least $1 - \delta$:
\begin{equation}
\label{eq:firstguarantee}
\L_{D_{T+1}} \left(\sum_{t = 1}^T w_th_t\right)\leq \sum_{t = 1}^T
w_tL(h_t(x_t), y_t) + \bar \Delta(\w, T) + M\| \w \|_2\sqrt{2\log \frac{1}{\delta}}.
\end{equation}
This proves the first statement of the theorem. To prove the second
claim, we will bound the empirical error in terms of the regret.  For
any $h^* \in H$, we can write \linebreak  using $\inf_{h \in H} \frac{1}{T}\sum_{t
  = 1}^T L(h(x_t), y_t) \leq \frac{1}{T} \sum_{t = 1}^T L(h^*(x_t),
y_t)$:
\begin{align*}
  & \sum_{t = 1}^T w_t L(h_t(x_t), y_t) - \sum_{t = 1}^T w_t L(h^*(x_t), y_t)\\
  & = \!\sum_{t = 1}^T \! \Big(w_t \!-\! \frac{1}{T} \Big) [L(h_t(x_t), y_t)
  \!-\! L(h^*(x_t), y_t)] \!+\! \frac{1}{T} \!\sum_{t = 1}^T [L(h_t(x_t), y_t) \!-\!
  L(h^*(x_t), y_t)] \\
  & \leq M\| \w - \u_0 \|_1 +\frac{1}{T}\sum_{t = 1}^T
  L(h_t(x_t),y_t)-\inf_h \frac{1}{T} \sum_{t = 1}^T L(h(x_t), y_t)\\
 & \leq M\| \w - \u_0 \|_1+\frac{R_T}{T}.
\end{align*}
Now, by definition of the infimum, for any $\e > 0$, there exists $h^*
\in H$ such that $\L_{D_{T+1}}(h^*) \leq \inf_{h \in H}
\L_{D_{T+1}}(h) + \e$. For that choice of $h^*$, in view of
\eqref{eq:firstguarantee}, with probability at least $1 - \delta/2$,
the following holds:
\begin{equation*}
\L_{D_{T+1}}(h) 
\leq \sum_{t = 1}^T w_t L(h^*(x_t), y_t) + M\| \w_ - \u_0
\|_1+\frac{R_T}{T} + \bar \Delta(\w, T)  + M\| \w \|_2\sqrt{2\log \frac{2}{\delta}}.
\end{equation*}
By the second statement of Lemma~\ref{lemma:13}, for any $\delta > 0$,
with probability at least $1 - \delta/2$,
\begin{equation*}
  \sum_{t = 1}^T w_t L(h^*(x_t),y_t)  \leq \L_{D_{T+1}}(h^*) + \bar \Delta(\w, T) + M\|
  \w
  \|_2\sqrt{2\log \frac{2}{\delta}}.
\end{equation*}
Combining these last two inequalities, by the union bound, with
probability at least $1 - \delta$, the following holds with $B(\w,
\delta) = M\| \w_ - \u_0 \|_1+\frac{R_T}{T} + 2 M\| \w \|_2\sqrt{2\log
  \frac{2}{\delta}}$:
\begin{align*}
\L_{D_{T+1}}(h) 
& \leq \L_{D_{T+1}}(h^*) + 2 \bar \Delta(\w, T) + B(\w, \delta)\\
& \leq \inf_{h \in H} \L_{D_{T+1}}(h) + \e + 2 \bar \Delta(\w, T) + B(\w, \delta).
\end{align*}
The last inequality holds for all $\e > 0$, therefore also for $\e = 0$ by
taking the limit.\qed
\end{proof}

\section{Algorithm}
\label{sec:algorithm}

The results of the previous section suggest a natural algorithm based
on the values of the discrepancy between distributions.  Let $(h_t)_{t
  = 1}^T$ be the sequence of hypotheses generated by an on-line
algorithm. Theorem~\ref{th:onlinetobatch} provides a learning
guarantee for any convex combination of these hypotheses. The convex
combination based on the weight vector $\w$ minimizing the bound of
Theorem~\ref{th:onlinetobatch} benefits from the most favorable
guarantee.  This leads to an algorithm for determining $\w$ based on
the following convex optimization problem:
\begin{align}
\label{pb:optimization}
\min_{\w} & \quad \lambda \| \w \|_2^2 +\sum_{t = 1}^T w_t \,
 (\dis_\cY(D_t,D_{T+1})+L(h_t(x_t),y_t))\\[-.25cm]
\text{subject to:} & \quad \Big( \sum_{t = 1}^T w_t = 1 \Big) \wedge (\forall t \in [1, T], w_t \geq 0),\nonumber
\end{align}
where $\lambda \geq 0$ is a regularization parameter.  This is a
standard QP problem that can be efficiently solved using a variety of
techniques and available software.

In practice, the discrepancy values $\dis_\cY(D_t, D_{T+1})$
are not available since they require labeled samples. But, in the
deterministic scenario where the labeling function $f$ is in $H$, we
have $\dis_\cY(D_t, D_{T+1}) \leq \dis(D_t, D_{T+1})$. Thus, the
discrepancy values $\dis(D_t, D_{T+1})$ can be used instead in our
learning bounds and in the optimization \eqref{pb:optimization}. This
also holds approximately when $f$ is not in $H$ but is close to some
$h \in H$.

As shown in \citep{MansourMohriRostamizadeh2009}, given two
(unlabeled) samples of size $n$ from $D_t$ and $D_{T + 1}$, the
discrepancy $\dis(D_t, D_{T + 1})$ can be estimated within
$O(1/\sqrt{n})$, when $\R_n(H_L) = O(1/\sqrt{n})$.  In many realistic
settings, for tasks such as spam filtering, the distribution $D_t$
does not change within a day. This gives us the opportunity to collect
an independent \emph{unlabeled} sample of size $n$ from each
distribution $D_t$. If we choose $n \gg T$, by the union bound, with
high probability, all of our estimated discrepancies will be within
$O(1/\sqrt{T})$ of their exact counterparts $\dis(D_t, D_{T+1})$.

Additionally, in many cases, the distributions $D_t$ remain unchanged
over some longer periods (cycles) which may be known to us. This in
fact typically holds for some tasks such as spam filtering, political
sentiment analysis, some financial market prediction problems, and
other problems. For example, in the absence of any major political
event such as a debate, speech, or a prominent measure, we can expect
the political sentiment to remain stable. In such scenarios, it
should be even easier to collect an unlabeled sample from each
distribution. More crucially, we do not need then to estimate the
discrepancy for all $t \in [1, T]$ but only once for each cycle.

\subsection{Experiments}
 
Here, we report the results of preliminary experiments demonstrating
the performance of our algorithm. We tested our algorithm on synthetic
data in a regression setting.  The testing and training data were
created as follows: instances were sampled from a two-dimensional
Gaussian random variables $\mathcal{N}(\Mu_t, 1)$. The objective
function at each time was given by $y_t = \w_t \cdot \x_t$. The weight
vectors $\w_t$ and mean vectors $\Mu_t$ were selected as follows:
$\Mu_t = \Mu_{t-1} + \mathbf{U}$ and $\w_t = R_{\theta} \w_{t-1}$,
where $\mathbf{U}$ is the uniform random variable over $[-.1, +.1]^2$
and $R_\theta$ a rotation of magnitude $\theta$ distributed uniformly
over $(-1,1)$.  We used the Widrow-Hoff algorithm \cite{WidrowHoff88}
as our base on-line algorithm to determine $h_t$. After receiving $T$
examples, we tested our final hypothesis on $100$ points taken from
the same Gaussian distribution $\mathcal{N}(\Mu_{T+1}, 1)$.  We ran
the experiment $50$ times for different amounts of sample points and
took the average performance of our classifier. For these experiments,
we are considering the ideal situation where the discrepancy values
are given.

\begin{figure}[t]
\begin{center}
\includegraphics[scale=.42]{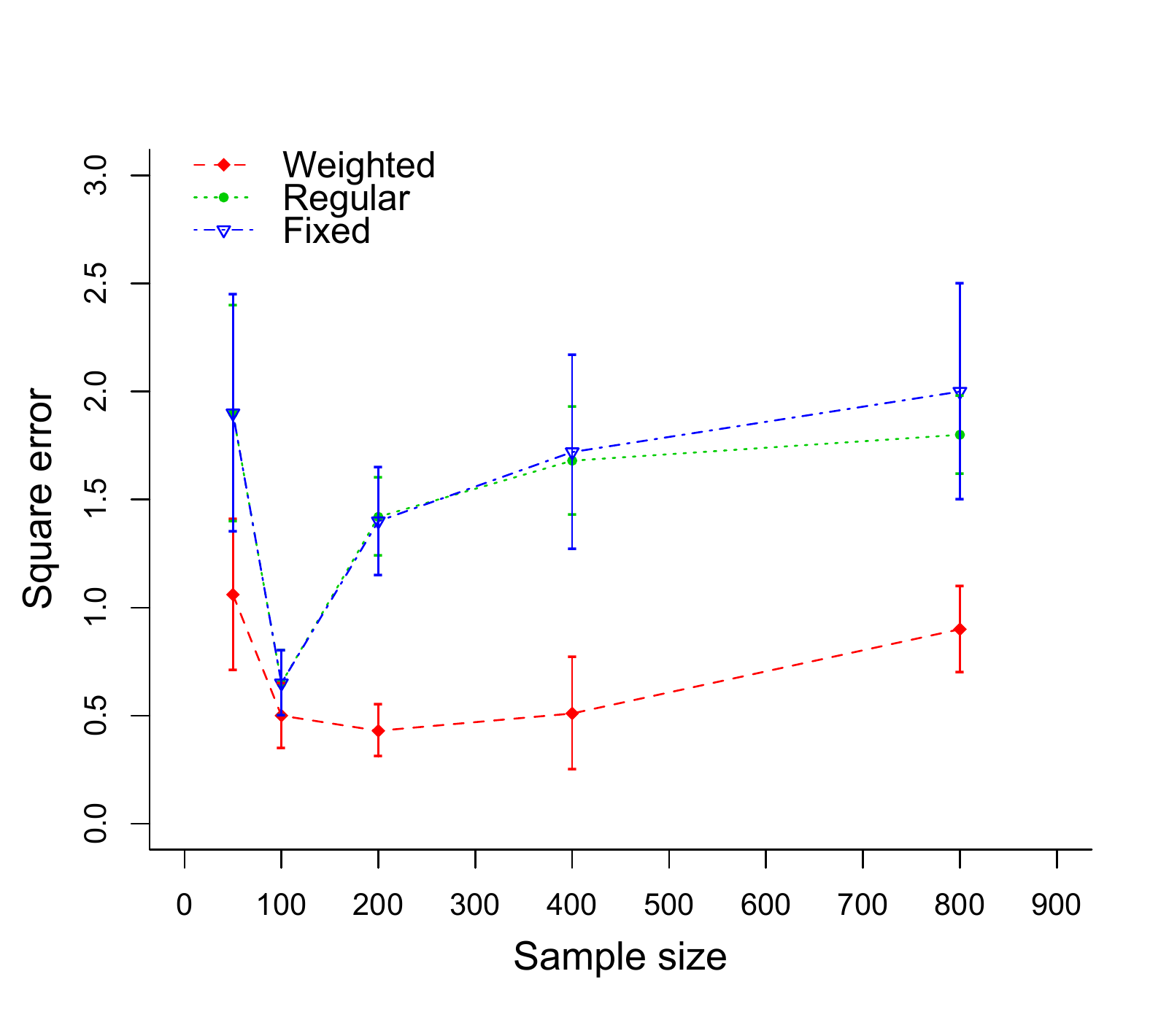}
\vspace{-.5cm}
\caption{
  \label{fig:performance} Comparison of the performance of three
  algorithms as a function of the sample size $T$. {\tts Weighted}
  stands for the algorithm described in this paper, {\tts Regular} for
  an algorithm that averages over all the hypotheses, and {\tts Fixed} 
  for the algorithm that averages only over the last 100 hypotheses.}
\end{center}
\vspace{-.5cm}
\end{figure}

We compared the performance of our algorithm with that of the
algorithm that (uniformly) averages all of the hypotheses and with
that of the algorithm that averages only the last 100 hypotheses
generated by the perceptron algorithm.  Figure~\ref{fig:performance}
shows the results of our experiments in the first setting.  Observe
that the error increases with the sample size. While the analysis of
Section~\ref{sec:pac} could provide an explanation of this phenomenon
in the case of the uniform averaging algorithm, in principle, it does
not explain why the error also increases in the case of our
algorithm. The answer to this can be found in the setting of the
experiment. Notice that the Gaussians considered are moving their
center and that the squared loss grows proportional to the radius of
the smallest sphere containing the sample. Thus, as the number of
points increases, so does the maximum value of the loss function in
the test set.  Nevertheless, our algorithm still outperforms the other
two algorithms. It is worth noting that the accuracy of our algorithm
can drastically change of course depending on the choice of the online
algorithm used.

\section{Conclusion}

We presented a theoretical analysis of the problem of learning with
drifting distributions in the batch setting. Our learning guarantees
improve upon previous ones based on the $L_1$ distance, in some cases
substantially, and our proofs are simpler and concise. These bounds
benefit from the notion of discrepancy which seems to be the natural
measure of the divergence between distributions in a drifting
scenario. This work motivates a number of related studies, in
particular a discrepancy-based analysis of the scenario introduced by
\cite{CrammerEven-DarMansourWortman2010} and further improvements of
the algorithm we presented, in particular by exploiting the specific
on-line learning algorithm used. 

\section{Acknowledgments}

We thank Yishay Mansour for discussions about the topic of this paper.

\bibliographystyle{splncs}
{\small 
\bibliography{drift}
}

\end{document}